%% file: NeurIPS-Non-Markovian.tex
\pdfoutput=1
\documentclass{article}

\usepackage[preprint,nonatbib]{neurips_2022}

\usepackage[utf8]{inputenc} 
\usepackage[T1]{fontenc}    
\usepackage{url}            
\usepackage{booktabs}       
\usepackage{amsfonts}       
\usepackage{nicefrac}       
\usepackage{microtype}      
\usepackage{xcolor}
\usepackage{xparse}
\usepackage{multicol}
\usepackage{hyperref}       
\usepackage{tikz}
\usetikzlibrary{arrows,automata}
\tikzset{terminal state/.style={draw,rectangle,minimum size=.3in}}

\usepackage{graphicx}
\usepackage{wrapfig}
\usepackage{subcaption}
\usepackage[ruled]{algorithm2e}
\usepackage{amssymb,amsmath,mathtools,dsfont}
\usepackage{amsthm}
\usepackage{thm-restate}
\usepackage{enumerate}
\newtheorem{theorem}{Theorem}
\newtheorem*{theorem*}{Theorem}

\newtheorem{defi}{Definition}
\newtheorem{prop}{Proposition}
\newtheorem{corollary}{Corollary}

\usepackage{bm}

\def\defeq{\dot=}

\newcommand*{\mytop}{\mathrel{\scalebox{0.5}{$\top$}}}

\newcommand*{\myplus}{\mathrel{\scalebox{0.5}{$+$}}}

\newcommand\occ{\stackrel{\mathclap{\normalfont\mbox{\scalebox{0.4}{\textsc{occ}}}}}{\sim}}

\tikzstyle{startstop} = [rectangle, rounded corners, minimum width=3cm, minimum height=1cm,text centered, draw=black, fill=red!30]
\tikzstyle{io} = [trapezium, trapezium left angle=70, trapezium right angle=110, minimum width=3cm, minimum height=1cm, text centered, draw=black, fill=blue!30]
\tikzstyle{process} = [rectangle, minimum width=3cm, minimum height=1cm, text centered, draw=black, fill=orange!30]
\tikzstyle{decision} = [diamond, minimum width=3cm, minimum height=1cm, text centered, draw=black, fill=green!30]
\tikzstyle{arrow} = [thick,->,>=stealth]

\DeclareMathOperator*{\argmax}{argmax}

\theoremstyle{definition}

\definecolor{color1}{RGB}{215,25,28}
\definecolor{color2}{RGB}{26,150,65}
\definecolor{color3}{RGB}{200,10,10}

\title{Non-Markovian policies occupancy measures}

\author{
Romain Laroche* \\
Microsoft Research Montr\'eal\\
\texttt{rolaroch@microsoft.com}
\And
Remi Tachet des Combes* \\
Microsoft Research Montr\'eal \\
\texttt{retachet@microsoft.com}
\And
Jacob Buckman \\
MILA, McGill University, Montr\'eal \\
\texttt{jacobbuckman@gmail.com}
}

\begin{document}

\maketitle

\begin{abstract}
A central object of study in Reinforcement Learning (RL) is the Markovian policy, in which an agent’s actions are chosen from a memoryless probability distribution, conditioned only on its current state. The family of Markovian policies is broad enough to be interesting, yet simple enough to be amenable to analysis. However, RL often involves more complex policies: ensembles of policies, policies over options, policies updated online, etc. Our main contribution is to prove that the occupancy measure of any non-Markovian policy, \textit{i.e.}, the distribution of transition samples collected with it, can be equivalently generated by a Markovian policy.

This result allows theorems about the Markovian policy class to be directly extended to its non-Markovian counterpart, greatly simplifying proofs, in particular those involving replay buffers and datasets. We provide various examples of such applications to the field of Reinforcement Learning.
\end{abstract}

\section{Introduction}

\emph{Reinforcement learning (RL)} is a popular and powerful theoretical framework for computational decision-making \cite{suttonbarto}, with many impressive accomplishments \cite{dqn, alphago}. A central object of study in the field is the \emph{Markovian policy}, in which an agent’s actions are chosen from a \textit{memoryless} probability distribution, \textit{i.e.}, are conditioned only on its current state. The family of Markovian policies is broad enough to be interesting, yet simple enough to be amenable to analysis. For example, every MDP admits an optimal Markovian policy \cite{suttonbarto}, and it is possible to guarantee monotonic improvement when moving between Markovian policies \cite{kakadelangford,Laroche2021}.

However, RL settings and algorithms often also involve \emph{non-Markovian policies}, which may choose different probability distributions of actions in the same state depending on additional context. For example, non-Markovian policies are encountered in \emph{Offline RL}~\cite{Levine2020}, where the agent is not given the opportunity to interact with the environment at all, but instead must learn from a dataset of trajectories collected by an arbitrary set of policies. Other examples are algorithms that use replay buffers~\cite{dqn}, update online~\cite{dqn}, and/or allow sub-policy-switching, such as in the Semi MDP framework, options, or hierarchical policies~\cite{barto2003recent,nachum2018data,stolle2002learning,Sutton1999}.
Formal analysis of these settings is possible, but somewhat involved. A typical approach is to prove a result under the assumption that trajectories are collected with a Markovian policy~\cite{Dias2019b,awr}.

Our main contribution is to show a certain form of equivalence between Markovian policies and collections of non-Markovian policies. Concretely, we prove that the \textit{occupancy measure}~\cite{csaba} (also called state-action visits~\cite{Sutton1998}, or distribution~\cite{pmlr-v108-cheng20b,conf/icml/SilverLHDWR14}) of any non-Markovian policy can be equivalently obtained by a Markovian policy.
\begin{theorem}[Restricted version of Theorem~\ref{thm:main}]
    Let $m$ be an MDP with finite state and action spaces, a discount factor $\gamma < 1$, and let $\pi$ be a policy. Then, there exists a Markovian policy $\tilde{\pi}$ that has the same occupancy measure in $m$ as $\pi$.
    \label{thm:informal}
\end{theorem}

Theorem 1 is stated in unpublished lecture notes~\cite{csaba} and to the best of our knowledge has not been formally proved. In this paper, we provide a generalization of the result, covering all Markov Decision Processes (including infinite state and action spaces and absence of discounting), and we give a necessary and sufficient condition for the theorem to hold, namely the $\sigma$-finiteness of the occupancy measure of the non-Markovian policy.

Since the performance of a policy only depends on its occupancy measure, and not on the actual trajectories, the first implication of the theorem is that, for any policy, there always exists a Markovian policy with the same performance.
More importantly, we argue that our result allows many theorems about the Markovian policy class to be directly extended to its non-Markovian counterpart, greatly simplifying proofs, in particular those involving replay buffers~\cite{awr} and datasets~\cite{Laroche2017safe,shi2022pessimistic,yin2021near}.



The paper is organized as follows. First, we give a minimal example illustrating the state-action visits equivalence and the construction of the Markovian policy (Section \ref{sec:intuition}).
We then provide the background and notations used in the paper (Section \ref{sec:background}). Next, we formally introduce the concept of occupancy measure (Section \ref{sec:occupancy}). Section \ref{sec:theory} details and discusses our various results. In Section \ref{sec:application}, we apply our main theorem to two domains: Offline RL and experience replays. Furthermore, we speculatively motivate its impact by enumerating fields of research where it could be potentially useful. Section~\ref{sec:proofs} details the proofs of our main results. Finally, Section \ref{sec:conclusion} concludes the paper.

\begin{wrapfigure}{r}{0.31\textwidth}
        \vspace{-60pt}
        \begin{center}
            \scalebox{1}{
                \begin{tikzpicture}[->, >=stealth', scale=1.5 , semithick, node distance=2cm]
                    \tikzstyle{every state}=[fill=white,draw=black,thick,text=black,scale=1]
                    \node[state]    (x0)                {$s$};
                    \node[state,accepting]    (xf)[right of=x0]   {$s_{f}$};
                    \path
                    (x0) edge[above]    node{$a_2$}     (xf)
                    (x0) edge[out=135, in=225, loop, right]    node{$a_1$}     (x0);
                \end{tikzpicture}
            }
            \vspace{-40pt}
            \caption{Minimal MDP such that $a_1$ loops and $a_2$ is terminal.}
            \label{fig:minimalMDP}
        \end{center}
\end{wrapfigure}
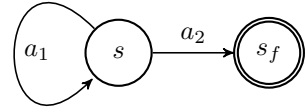
\subsection{Illustrative example}
\label{sec:intuition}
    We consider the MDP $m$ represented in Figure \ref{fig:minimalMDP}, with a single state $\mathcal{S} = \{s\}$ two actions $\mathcal{A}= \{a_1,a_2\}$, and such that $p(s|s,a_1)=1$ and $a_2$ terminates the episode. For a fixed $n\geq 1$, we consider the following deterministic policy: $\pi(a_1|s,t<n)=1$ and $\pi(a_2|s,t=n)=1$. All its trajectories are therefore the same: it performs $a_1$ $n$ times, then $a_2$ which terminates the trajectory. Fundamentally, $\pi$ is non-Markovian since its actions depends on the timestep $t$. It is direct to observe that the expected discounted state and state-action visits (denoted $\mu_{\gamma}^{\pi}(s)$ and $\mu_{\gamma}^{\pi}(s,a)$ respectively) are:
    \begin{align*}
        \mu_{\gamma}^{\pi}(s) &= \sum_{t=0}^{n} \gamma^t  & \mu_{\gamma}^{\pi}(s,a_1) &= \sum_{t=0}^{n-1} \gamma^t & \mu_{\gamma}^{\pi}(s,a_2) &= \gamma^{n}.
    \end{align*}
        We define a Markovian policy $\tilde{\pi}$, equal to the expected behavior of $\pi$ in $s$ over all possible trajectories:
    \begin{align}
        \tilde{\pi}(a_1|s) &\doteq\frac{\mu_{\gamma}^{\pi}(s,a_1)}{\mu_{\gamma}^{\pi}(s)}= \frac{\sum_{t=0}^{n-1} \gamma^t}{\sum_{t=0}^{n} \gamma^t} & \tilde{\pi}(a_2|s) &\doteq\frac{\mu_{\gamma}^{\pi}(s,a_2)}{\mu_{\gamma}^{\pi}(s)}= \frac{\gamma^{n}}{\sum_{t=0}^{n} \gamma^t}.
    \end{align}
    Given the MDP structure, we have $\mu_{\gamma}^{\tilde{\pi}}(s,a_1) = \tilde{\pi}(a_1|s) + \tilde{\pi}(a_1|s) \; \gamma \mu_{\gamma}^{\tilde{\pi}}(s,a_1)$. Therefore:
    \begin{align*}
        \mu_{\gamma}^{\tilde{\pi}}(s,a_1) = \frac{\tilde{\pi}(a_1|s)}{1- \tilde{\pi}(a_1|s)\gamma }
        = \frac{\sum_{t=0}^{n-1} \gamma^t}{\sum_{t=0}^{n} \gamma^t- \gamma\sum_{t=0}^{n-1} \gamma^t} = \sum_{t=0}^{n-1} \gamma^t = \mu_{\gamma}^{\pi}(s,a_1).
    \end{align*}
    Similarly, $\mu_{\gamma}^{\tilde{\pi}}(s,a_2) = \tilde{\pi}(a_2|s) + \tilde{\pi}(a_1|s) \; \gamma \mu_{\gamma}^{\tilde{\pi}}(s,a_2)$, which gives:
    \begin{align*}
        \mu_{\gamma}^{\tilde{\pi}}(s,a_2) = \frac{\tilde{\pi}(a_2|s)}{1- \tilde{\pi}(a_1|s)\gamma } = \frac{\gamma^{n}}{\sum_{t=0}^{n} \gamma^t- \gamma\sum_{t=0}^{n-1} \gamma^t} = \gamma^{n} = \mu_{\gamma}^{\pi}(s,a_2).
    \end{align*}
    
    We see that $\tilde{\pi}$ has exactly the same state-action visits as $\pi$. Our main theorem states that, under mild assumptions, such a policy always exists. We wish to emphasize, however, that their trajectories distributions differ: all trajectories generated with $\pi$ have the same length, $n+1$, while the length of trajectories generated with $\tilde{\pi}$ follows a geometric law. It is also worth noticing that $\tilde{\pi}$ depends on the choice of the discount factor, but that, for any discount factor, there will exist a Markovian policy equivalent to $\pi$ in terms of state-action occupancy.
    
\subsection{Background and notations}
\label{sec:background}
In this section, we introduce the definitions and notations required to state our result in its most general form. Capitalized letters denote random variables. Sets are denoted by calligraphic letters, and subsets by lower-case greek letters, except for $\mu$, $\gamma$, and $\pi$, which are well-established notations for measures, discount factors, and policies respectively. If not mentioned otherwise, any set $\mathcal{X}$ is equipped with a $\sigma$-algebra $\Sigma_\mathcal{X}$ and a measure $\mu_\mathcal{X}$. We will write $\mathcal{P}(\Sigma_\mathcal{X})$ for the set of probability measures over $\Sigma_\mathcal{X}$. 

Typically, for a countable set $\mathcal{X}$, $\Sigma_\mathcal{X}$ is the set of all subsets of $\mathcal{X}$, written $2^\mathcal{X}$ (also called its powerset), and $\mu_\mathcal{X}$ is defined as the counting measure, \textit{i.e.}, $\mu_\mathcal{X}(\xi)$ is the number of elements in $\xi$ for any $\xi\in\Sigma_\mathcal{X}$. Typically, for $\mathcal{X}\subset\mathbb{R}^n$, $\Sigma_\mathcal{X}$ is the Lebesgue $\sigma$-algebra and $\mu_\mathcal{X}(\xi)$ the Lebesgue measure. 

A Markov Decision Process (MDP) is a tuple $m=\langle \mathcal{S}, \mathcal{A},p_0,p,r,\gamma \rangle$, where $\mathcal{S}$ is the state space, $\mathcal{A}$ the action space, $p_0(\cdot)\in\mathcal{P}(\Sigma_{\mathcal{S}}\times\{\emptyset,\{s_f\}\})$ denotes the initial state distribution, $p(\cdot|s,a)\in\mathcal{P}(\Sigma_{\mathcal{S}}\times\{\emptyset,\{s_f\}\})$ is the transition kernel, $s_f\notin\mathcal{S}$ denotes the final state where episode termination happens, $r(s,a)\in[-r_{\mytop},r_{\mytop}]$ is the bounded stochastic reward function, and $\gamma\in[0,1]$ denotes the discount factor.

\begin{defi}[Policy]
    A policy $\pi$ represents any function mapping its trajectory history $h_{t} = \langle s_0,a_0,r_0\dots,s_{t-1},a_{t-1},r_{t-1},s_t \rangle$ to a distribution over actions: $\pi(\cdot|h_t)\in \mathcal{P}(\Sigma_\mathcal{A})$. Let $\Pi$ denote the space of policies. \label{def:policy}
\end{defi}

This functional definition considers the policy as a black box: its inner workings do not matter as long as they do not manifest in the environment. It is a compact, necessary and sufficient, way of describing any policy\footnote{Although, it is arguably an inefficient way of designing/implementing one.}: either two policies differ for some (accessible) history and they are distinguishable, or they do not differ anywhere and they are indistinguishable. Thus, this is a fully general definition\footnote{In order to have a well-defined occupancy measure, we must restrict ourselves to policies that reset their memory at the start of every trajectory. A thorough discussion is provided at the end of Section \ref{sec:theory} regarding this.}: any behavior can be implemented by acting according to a member of $\Pi$. 

\begin{defi}[Markovian policy]
    Policy $\pi$ is said to be Markovian if its action probabilities only depend on the current state $s_t$: $\pi(\cdot|h_t)=\pi(\cdot|s_t)\in \mathcal{P}(\Sigma_\mathcal{A})$. Otherwise, policy $\pi$ is non-Markovian. We let $\Pi_\textsc{m}$ denote the space of Markovian policies. We let $\Pi_\textsc{dm}$ denote the space of deterministic Markovian policies, \textit{i.e.}, the set of Markovian policies such that $\pi(\cdot|s)$ is a Dirac distribution in some action $a_s$ for any state $s\in\mathcal{S}$. \label{def:markovpolicy}
\end{defi}

Because of the Markovian property of the MDP environment, Markovian policies are often a sufficiently broad set to solve the RL problem. In particular, there always exists an optimal policy that is deterministic Markovian, and the Markovian policy space happens to be convenient to navigate smoothly between deterministic Markovian policies. 

Next, we shall also need various basic measure theory concepts that we recall here~\cite{feller1}.

\begin{defi}[$\sigma$-finiteness] A measure $\mu$ on a measurable space $(\mathcal{X}, \Sigma_\mathcal{X})$ is $\sigma$-finite if there exists a sequence $(\xi_n)_{n \in \mathbb{N}} \in \Sigma_\mathcal{X}^{\mathbb{N}}$ such that $\mathcal{X} = \cup_{n=0}^{\infty} \; \xi_n$ and $\mu(\xi_n) < +\infty$ for all $n \in \mathbb{N}$.
\end{defi}

\begin{defi}[Radon-Nikodym derivative] Let $\mu$ and $\nu$ denote two $\sigma$-finite measures where $\nu$ is absolutely continuous with respect to $\mu$ (\textit{i.e.}, $\mu(\xi) = 0 \implies \nu(\xi) = 0$). There exists a function $f: \mathcal{X} \to [0,+\infty]$ such that for all $\xi \in \Sigma_\mathcal{X}$,
\begin{align}
    \nu(\xi) = \int_\xi f(x) \mu(dx).
    \label{eq:RN}
\end{align}

Any function $f$ verifying~\ref{eq:RN} is called a Radon-Nikodym derivative and is denoted $\frac{d\nu}{d\mu}$. Two functions $f_1,f_2$ that verify~\ref{eq:RN} are equal up to a $\mu$-null set, \textit{i.e.}, $\mu\left(\{x \text{ s.t. } f_1(x)\neq f_2(x)\}\right)=0$.
\end{defi}

\subsection{Occupancy measures}
\label{sec:occupancy}
We now have the machinery required to introduce our main object of study: the occupancy measure.

\begin{defi}[Occupancy]
    Given measurable subsets of the state and action spaces, $\sigma\in\Sigma_{\mathcal{S}}, \alpha\in\Sigma_{\mathcal{A}}$, the occupancy $\mu^{\pi}_{\gamma}(\sigma,\alpha)$ of a policy $\pi\in\Pi$ in an MDP $m=\langle \mathcal{S}, \mathcal{A},p_0,p,r,\gamma \rangle$ is the expected discounted number of visits of a state-action pair $(s,a)\in\sigma\times\alpha$ occurring during a trajectory:
    \begin{align}
        \mu^{\pi}_{\gamma}(\sigma,\alpha) \coloneqq \mathbb{E}\left[\sum_{t=0}^\infty \gamma^t\mathds{1}\left(S_t \in \sigma\right)\times\mathds{1}\left(A_t \in \alpha\right) \bigg|\!\!\begin{array}{l}
    S_0\sim p_0(\cdot), A_t \sim \pi(\cdot|H_{t}),\\S_{t+1}\sim p(\cdot|S_t,A_t)\end{array}\right] .
        \label{eq:occupancy-measure}
    \end{align}
    \label{def:occupancy-measure}
\end{defi}
We will use the conventions that $\mu^{\pi}_{\gamma}(\sigma)\doteq\mu^{\pi}_{\gamma}(\sigma,\mathcal{A})$, and with finite state-action sets $\mu^{\pi}_{\gamma}(s,a)\doteq\mu^{\pi}_{\gamma}(\{s\},\{a\})$. Note that Definition \ref{def:occupancy-measure} holds both for discrete and continuous state and action spaces.

We start by establishing that the occupancy as defined in Eq. \eqref{eq:occupancy-measure} is well-defined and is a measure for any policy $\pi$ and any MDP $m$, this will allow us to leverage standard results from measure theory.
\begin{restatable}[\textbf{Occupancy is a measure}]{theorem}{measure}
    Let $\pi\in\Pi$ be any policy as defined in \ref{def:policy}, then, $\mu_{\gamma}^\pi$ is well-defined on $\mathbb{R}^{\myplus}\cup\{+\infty\}$ and is a measure.
    \label{thm:measure}
\end{restatable}

We defer the proof of this result to Appendix~\ref{app:occupancymeasure}. The second interesting property of $\mu_{\gamma}^\pi$ concerns the discounted return $\rho^{\pi}_{\gamma}$ of $\pi$:
\begin{align}
    \rho^{\pi}_{\gamma} \coloneqq& \mathbb{E}\left[\sum_{t=0}^\infty \gamma^t R_t \bigg|\!\!\begin{array}{l}
S_0\sim p_0(\cdot), A_t \sim \pi(\cdot|H_{t}),\\R_t\sim r(S_t,A_t),S_{t+1}\sim p(\cdot|S_t,A_t)\end{array}\right].
    \label{eq:performance}
\end{align}

\begin{restatable}{lemma}{performance}
\label{lem:performance}
If $\rho^{\pi}_{\gamma}$ exists, then it is uniquely characterized by $\mu_{\gamma}^\pi$: $\rho^{\pi}_{\gamma} = \int_{\mathcal{S}} \int_{\mathcal{A}} \mathbb{E}\left[r(s,a)\right]\mu_\gamma^\pi(ds,da)$.
\end{restatable}

We defer the proof of this result to Appendix~\ref{app:performanceequivalence}. A general equivalence in terms of value is harder to make as there is no clear definition of a state marginalized value function for non Markovian policies.

\section{The occupancy measure equivalence}
Let us now state our main theorem and discuss its implications.

\label{sec:theory}
\begin{restatable}[\textbf{State-action occupancy equivalence}]{theorem}{mainthm}
    Let $\pi$ be a policy with $\sigma$-finite occupancy measure $\mu^\pi_{\gamma}$. For any measurable $\alpha\subseteq\mathcal{A}$, we define $\tilde{\pi}$ as the Radon-Nikodym derivative:
    \begin{align}
        \tilde{\pi}(\alpha | s) \coloneqq& \frac{d \mu^\pi_{\gamma}(\cdot, \alpha)}{d \mu^\pi_{\gamma}(\cdot )}(s),
        \label{eq:ts2stoch}
    \end{align}
    where $\mu^\pi_{\gamma}(\cdot ,  \alpha)$ and $\mu^\pi_{\gamma}(\cdot )$ are seen as measures on $\mathcal{S}$. The following statements hold: 
    \begin{itemize}
        \item $\tilde{\pi}(\alpha | s)$ exists and is a probability measure on $\mathcal{A}$ for any $s \in \mathcal{S}$, \textit{i.e.}, a Markovian policy.
        \item $\tilde{\pi}$ admits a $\sigma$-finite occupancy measure, and $\mu^{\tilde{\pi}}_{\gamma}=\mu^\pi_{\gamma}$.
    \end{itemize}
    \label{thm:main}
\end{restatable}

\textbf{Remark:} Note that $\tilde{\pi}$ is uniquely defined up to a $\mu^\pi_{\gamma}(\cdot )$-null set. Further characterizing it is unimportant as a set of states with null occupancy measure will almost surely never be visited.

In two very generic settings, the Radon-Nikodym derivative $\tilde{\pi}$ can be explicitly characterized.

\begin{corollary}\label{thm:main2}
When both $\mathcal{S}$ and $\mathcal{A}$ are finite, we let $\mu_\gamma^\pi(s,a)$ denote the standard state-action visitation of the pair $(s,a) \in\mathcal{S}\times\mathcal{A}$ under policy $\pi$. Then, we simply have $\tilde{\pi}(a | s) \coloneqq \frac{\mu_\gamma^\pi(s, a)}{\mu_\gamma^\pi(s)}$, and similarly $\mu_\gamma^{\tilde{\pi}}(s,a) = \mu_\gamma^\pi(s,a)$.
\end{corollary}
Corollary~\ref{thm:main2} covers for instance the illustration provided in Section~\ref{sec:intuition}.

\begin{corollary}\label{thm:main3}
When $\mathcal{S}$ and/or $\mathcal{A}$ are infinite, let $d_\gamma^\pi(s,a)$ denote the standard state-action visitation density of the pair $(s,a) \in\mathcal{S}\times\mathcal{A}$ and assume its existence. Then, we simply have $\tilde{\pi}(\alpha | s) \coloneqq \int_\alpha\frac{d_\gamma^\pi(s, a)}{d_\gamma^\pi(s)}da$, and by abuse of notation $\tilde{\pi}(a | s) \coloneqq \frac{d_\gamma^\pi(s, a)}{d_\gamma^\pi(s)}$. Here as well, $d_\gamma^{\tilde{\pi}}(s,a) = d_\gamma^\pi(s,a)$.
\end{corollary}

Corollary~\ref{thm:main3} covers exclusively-continuous state and action spaces such as certain robotic manipulation tasks and some MuJoCo environments~\cite{todorov2012mujoco}. We note that Theorem \ref{thm:main} is more general than Corollaries \ref{thm:main2} and \ref{thm:main3} combined as the state-action visitation density in infinite state-action space is not always defined (there may be Dirac points). We now discuss the theorem in details.

\textbf{Policy performance:} The first trivial implication of Theorem~\ref{thm:main}, combined with Lemma~\ref{lem:performance}, is that for any policy $\pi$, there always exists a Markovian policy with the same performance.

\begin{restatable}{prop}{perfequiv}
    Under suitable existence assumptions, $\rho^{\pi}_{\gamma} = \rho^{\tilde{\pi}}_{\gamma}$.
    \label{prop:perf-equiv}
\end{restatable}

\textbf{Universality:} Our result is universal: it applies to any MDP $m$, any policy $\pi$, and any discount factor $\gamma$, as long as the occupancy measure of $\pi$ is $\sigma$-finite. The $\sigma$-finiteness condition is not an artifact of the proof technique, as illustrated by the following proposition (proof in Appendix \ref{app:infinite_occupancy_cex}).

\begin{restatable}{prop}{sigmainfinite}
    If $\pi$ has a $\sigma$-infinite occupancy measure, $\tilde{\pi}$ may be undetermined and there may not be any Markovian policy with the same occupancy measure as $\pi$.
    \label{prop:sigma-infinite}
\end{restatable}

\textbf{Idempotence:} Equation \eqref{eq:ts2stoch} may be interpreted as an operator over policies: $\tilde{\pi}=\mathcal{R}\pi$ up to a $\mu^\pi_{\gamma}(\cdot )$-null set. Proposition~\ref{prop:idempotence} proves that this operator is idempotent: $\mathcal{R}\mathcal{R}\pi=\mathcal{R}\pi$ (still up to a $\mu^\pi_{\gamma}(\cdot )$-null set). In other words, $\mathcal{R}$ is a projection from the policy space $\Pi$ to the Markovian policy space $\Pi_\textsc{m}$. This can also be seen through the lens of the following pseudo-metric on $\Pi$: $d(\pi_1, \pi_2) = TV(\mu^{\pi_1}_{\gamma},\mu^{\pi_2}_{\gamma})$ (where $TV$ denotes the total variation). Furthermore, one can define the occupancy equivalence relation: $\pi_1\occ\pi_2$ if policies $\pi_1$ and $\pi_2$ are occupancy equivalent, meaning that $\mu^{\pi_1}_{\gamma}=\mu^{\pi_2}_{\gamma}$ and implying that $\mathcal{R}\pi_1=\mathcal{R}\pi_2$ up to a $\left(\mu^{\pi_1}_{\gamma}(\cdot )+\mu^{\pi_2}_{\gamma}(\cdot )\right)$-null set (proof in Appendix \ref{app:idempotence}). 

\begin{restatable}{prop}{idempotence}
    If $\pi$ is Markovian with a $\sigma$-finite occupancy measure, then $\tilde{\pi}=\pi$, where equality is up to a $\mu^\pi_{\gamma}(\cdot )$-null set.
    \label{prop:idempotence}
\end{restatable}

\textbf{Trajectory distribution:} As already noted in Section \ref{sec:intuition}, the same occupancy measure does not imply the same trajectory distribution, which may have a specific role in non-bootstrapping algorithms. For instance, Decision Transformers~\cite{Chen2021,Furuta2022,Emmons2022} learn the return-conditional distribution of actions in each state, and then define a policy by sampling from the distribution of actions that receive high return in each state. It is therefore a trajectory-based algorithm as opposed to more classical RL bootstrapping approaches that are sample-based and can fully take advantage of Theorem \ref{thm:main}. Nevertheless, it is possible to prove the following result, stating that even though trajectory distributions are not equal for $\pi$ and $\tilde{\pi}$, any trajectory generated by $\pi$ has a non-zero probability under $\tilde{\pi}$ (proof in Appendix \ref{app:absolutelycontinuous}).
\begin{restatable}{prop}{actraj}
    For any $t \geq 0$, we let $\tau_t^\pi$ denote the distribution on $(\mathcal{S} \times \mathcal{A})^t$ induced by executing $\pi$ $t$ times in the environment, starting from $p_0$. Then, $\tau_t^\pi$ is absolutely continuous with respect to $\tau_t^{\tilde{\pi}}$.
    \label{prop:ac-traj}
\end{restatable}

\textbf{$\sigma$-finiteness of the occupancy measure:} Theorem \ref{thm:main} applies only if the policy's occupancy measure is $\sigma$-finite. It is trivially verified if $\gamma<1$. The following proposition gives another sufficient condition, we also prove in Apprendix \ref{app:finiteness} that it does not hold if finiteness is replaced with $\sigma$-finiteness.
\begin{restatable}{prop}{finiteness}
    If all deterministic Markovian policies have finite occupancy measures, any policy $\pi$ admits a finite occupancy measure.
    \label{prop:finiteness}
\end{restatable}

\textbf{Policies with inter-episode memory:} The occupancy measure theory does not concern policies that carry memory from one episode to another since such policies do not admit occupancy measures that are constant across episodes. Still, oftentimes datasets and replay buffers are collected with a learning algorithm, \textit{i.e.}, a policy $\pi$ that is updated across time. Let us consider the recording of the $N$ policies $\{\pi_i\}_{i\in[N]}\in\Pi^N$ that were used in each individual episode:
\begin{align}
    \pi_{i}(\cdot\;|\;h)\doteq \pi\left(\cdot\;|\;h\cup h_{\tau_{i-1}}\cup \dots \cup h_{\tau_1}\right),
\end{align}
where $h$ is the current trajectory history and $h_{\tau_i}$ denotes the recorded history of the $i$\textsuperscript{th} trajectory. Then, Theorem \ref{thm:main} may be applied on the inter-episode-policy $\bar{\pi}\doteq \mathcal{U}(\{\pi_i\}_{i\in[N]})$ that uniformly samples at the start of each trajectory a policy among the $N$ policies, allowing us to conclude that the occupancy in the dataset may still be reproduced by a single Markovian policy. Nevertheless, we stress again the fact that this occupancy is not connected to that of $\pi$, since the latter is undetermined.

\textbf{Non-Markovian policies usefulness:} The occupancy equivalence does not deny the algorithmic interest of non-Markovian policies, as their existence may be entailed by the problem setting (Offline RL or replay buffers), and as they may prove their usefulness by allowing to inject inductive bias, such as options, or by generating diverse behaviors with an ensemble of agents. On the contrary, we view Theorem \ref{thm:main} as a tool allowing to carry out theoretical grounding usually restricted to Markovian policies to non-Markovian policies. In the next section, we showcase various such applications.

\section{Direct applications}
\label{sec:application}
In this section, we motivate the impact of Theorem \ref{thm:main} by identifying several theoretical works that it readily allows to improve. In Subsection \ref{sec:spibb}, we argue it allows one to extend most of the theoretical results in the space of Offline RL that rely on behavioral cloning and for which previous guarantees only apply to data generated with a Markovian behavior policy. In particular, we derive the formal proof of the SPIBB family guarantees~\cite{Dias2019b} in the case of non-Markovian behavior policies. In Subsection \ref{sec:replaybuffer}, we further show that it is useful in the online RL setting, for algorithms which utilize a replay buffer. In particular, our result simplifies the derivation of the AWR \cite{awr} algorithm. Finally, Subsection \ref{sec:speculative} enumerates several RL research areas where non-Markovian policies are used and where, therefore, Theorem \ref{thm:main} may have some speculative impact.

\subsection{Offline RL: SPIBB theory}
\label{sec:spibb}
\textbf{Context and setting:} Offline RL consists in training a policy on a fixed set of observations without access to the true environment~\cite{Lange2012,Levine2020}. Most algorithms and analyses either implicitly or explicitly make the assumption that the behavior policy $\beta$ that was used to collect data is unique and Markovian~\cite{buckman2020importance,Fujimoto2019,kumar2020conservative,Laroche2017safe,thomas2015high,Yu2021,yin2021near,shi2022pessimistic}. However, this is generally not true: in healthcare for instance, patients are often followed by different doctors/health centers with different policies. Furthermore, typical benchmarks for offline reinforcement learning are constructed from the experience replay of DQN runs \cite{fujimoto2019benchmarking,agarwal2020optimistic}, or via an amalgamation of expert policies \cite{fu2020d4rl}.

In this section, we study a particular algorithm family called SPIBB for Safe Policy Improvement with Baseline Bootstrapping~\cite{Laroche2017safe,Nadjahi2019,Satija2021,simao2019structure,Dias2019,scholl2022safe}. It consists in allowing policy change only when the change is sufficiently supported by the dataset. These algorithms rely on two components: a state-action uncertainty and an estimate of the behavior policy~\cite{Dias2019b}. We will focus here on the latter.

Estimating a non-Markovian policy faces the curse of dimensionality, thus such policies are often not treated. Theorem \ref{thm:main} proves that the main theoretical results in~\cite{Dias2019b} actually carry over to non-Markovian behavioural policies. Indeed, since there exists a Markovian policy that yields the same expected performance and occupancy measure, one may simply consider that the data was received from the induced Markovian policy and obtain the same performance properties. The remainder of the subsection formalizes this informal claim.

\textbf{Background:} Offline RL is, in essence, reinforcement learning from a fixed dataset of $N$ full trajectories: $\mathcal{D} = \langle x_k,a_k, r_k, x'_k,t_k\rangle_{k\in[|\mathcal{D}|]}$. The starting state of transition $k$ is $x_{k} \sim p_0(\cdot)$ if $t_{k} = 0$ and $x_{k} = x'_{k-1}$ otherwise, $a_k \sim \beta(\cdot|h_k)$ is the performed action,  $r_{k} \sim r(x_k, a_k)$ is the immediate reward, $x'_{k} \sim p(\cdot|x_k, a_k)$ is the reached state, and the trajectory-wise timestep is $t_k=0$ if the previous transition was final and $t_k=t_{k-1}+1$ otherwise.

SPIBB theory only applies to finite MDPs and we therefore limit ourselves to finite MDPs as well in our proof extension. SPIBB algorithms consider the following quantities:
\begin{itemize}
    \item The sample counts: $n_{s,a,s'} = \sum_{k} \mathds{1}(s_k=s,a_k=a,s'_k=s')$, $n_{s,a} = \sum_{s'\in\mathcal{S}} n_{s,a,s'}$, $n_{s} = \sum_{a\in\mathcal{A}} n_{s,a}$, and $n_{s,0} = \sum_{k} \mathds{1}(s_k=s,t_k=0)$.
    \item The Maximum Likelihood Estimate of the environment (MLE MDP)  $\hat{m}=\langle \mathcal{S},\mathcal{A},\hat{p}_0,\hat{p},\hat{r},\gamma\rangle$: $\hat{p}_0(s) = \frac{n_{s,0}}{N}$, $\hat{p}(s'|s,a) = \frac{n_{s,a,s'}}{n_{s,a}}$, and $\hat{r}(s,a) = \frac{\sum_{k} r_k \mathds{1}(s_k=s,a_k=a)}{n_{s,a}}$.
    \item The MLE of the behavior policy (MLE behavior): $\hat{\beta}(a|s) = \frac{n_{s,a}}{n_s}$.
\end{itemize}

SPIBB algorithms constrain the policy search to some safe subset in the Markovian policy space $\Pi_{\hat{\beta}}\subset\Pi_\textsc{m}$ around $\hat{\beta}$, in order to optimize the expected trajectory return $\rho_{\hat{m}}^\pi$ in the MLE MDP: $\pi_\textsc{spibb} = \argmax_{\pi\in\Pi_{\hat{\beta}}} \rho_{\hat{m}}^\pi$.

\textbf{Theorem extension:} We observe from its definition that $\hat{\beta}$ is an unbiased estimate of $\beta$ only if it is Markovian. The proof of Theorem 3.2 in \cite{Dias2019b} makes the following decomposition to control the policy improvement of target policy $\pi$ with respect to the true Markovian policy $\beta$ in the true MDP $m$: $\rho^\pi_m-\rho^{\beta}_m = \left(\rho^\pi_m - \rho^{\hat{\beta}}_m\right) + \left(\rho^{\hat{\beta}}_m -\rho^{\beta}_m\right)$,
where the first and second terms are respectively controlled with safe policy improvement guarantees from the SPIBB theorems and with a behavior policy estimate concentrability. In order to deal with non-Markovian behavior policies $\beta$, we enhance the decomposition as follows:
\begin{align}
     \rho^\pi_m-\rho^{\beta}_m &= \left(\rho^\pi_m - \rho^{\hat{\tilde{\beta}}}_m \right)+ \left(\rho^{\hat{\tilde{\beta}}}_m -\rho^{\tilde{\beta}}_m\right) + \left(\rho^{\tilde{\beta}}_m -\rho^{\beta}_m\right), \label{eq:newdev2}
\end{align}
where $\tilde{\beta}$ is the Markovian policy that admits the same occupancy as the true behavior policy $\beta$, following Theorem \ref{thm:main2} (SPIBB guarantees are restricted to finite state and action spaces), and where $\hat{\tilde{\beta}}$ is the MLE estimate of $\tilde{\beta}$:
    \begin{align}
        \tilde{\beta}(a|s) \coloneqq \frac{\mu_{\gamma}^{\beta}(s,a)}{\mu_{\gamma}^{\beta}(s)} \quad\quad\quad\text{and}\quad\quad\quad\hat{\tilde{\beta}}(a|s) = \frac{\sum_{k} \gamma^{t_k}\mathds{1}(s_k=s,a_k=a)}{\sum_{k} \gamma^{t_k}\mathds{1}(s_k=s)}.
    \label{eq:ts2stoch-spibb}
    \end{align}

The first and second terms of Equation \eqref{eq:newdev2} are controlled in the same manner as their counterparts in the initial decomposition, and Theorem \ref{thm:main} proves that the third term is 0. Note that our extension of the proof is agnostic to the technique used to control the first two terms. It simply converts, at no cost, any result obtained for Markovian policies to any policy admitting a $\sigma$-finite occupancy measure, which is always the case when collecting trajectories since these trajectories must terminate in finite time in expectation. Note also that the estimate of the behavior policy $\hat{\tilde{\beta}}$ is slightly different from $\hat{\beta}$, so that utilizing $\hat{\beta}$ computed with the undiscounted formula would induce a bias for non-Markovian behavior policies.


\subsection{Experience replays in AWR}
\label{sec:replaybuffer}

Many modern deep reinforcement algorithms collect data via adaptive non-Markovian policies, and store data from many recent trajectories in a buffer for training \cite{dqn}. Theoretical analysis of algorithms which utilize experience replay must explicitly handle these nuances, often at significant effort. An example of this can be seen in Advantage-Weighted Regression \cite{awr}, where the derivation of the algorithm is duplicated to justify the use of a replay buffer. Note that in the discussion below, we adapt their notations to be consistent with our results.

The result in their Appendix A gives a proof that when data is collected under a fixed Markovian policy $\beta$, the objective function described in their Section 3.1 converges to (a projection of) the policy $\pi$ with optimal improvement over $\beta$. Concretely, they show in equation (26) that an approximation to the expected improvement is:
\begin{align}
\int_{\mathcal{S}} \mu^{\beta}_\gamma(s) \int_\mathcal{A} \pi(a|s) [ G^{\beta,\gamma}_{s,a} - v^\beta_\gamma(s) ] da ds,\label{eq:awr1}
\end{align}
where $G^{\beta,\gamma}_{s,a}$ denotes the discounted return random variable and $v^\beta_\gamma(s)$ the state value under policy $\beta$, which is only defined when $\beta$ is Markovian. The authors then derive the AWR algorithm by maximizing this approximate expected improvement.
It is not immediate that the data in the replay buffer, which may contain a mixture of policies, possesses such a stationary distribution. For the case where the dataset consists of $k$ policies $\{ \pi_1, \pi_2, ... \pi_k \}$, with probabilities $\{ w_1, w_2, ... w_k \}$, they extend their result in Appendix B by showing that a similar approximation holds: 
\begin{align}
\sum_{i=1}^k w_i \int_{\mathcal{S}} \mu^{\pi_i}_\gamma(s) \int_\mathcal{A} \pi_i(a|s) [ G^{\pi_i,\gamma}_{s,a} - v^{\pi_i}_\gamma(s) ] da ds.
\end{align}
From there, they re-derive AWR. Interestingly, this approximation relies on a construction, given explicitly in their equation (41), which corresponds precisely to our definition of $\tilde{\pi}$ in their specific setting.

Their derivation in Appendix B could have been achieved more simply by combining our Theorem \ref{thm:main} with the argument in their Appendix A. Optimizing the provided objective while sampling from the replay buffer yields the policy with optimal improvement over the policy whose stationary distribution matches that of the replay buffer, and whose existence we have verified in Theorem \ref{thm:main}. In fact, this new version of their proof would be more powerful, as their argument in Appendix B requires each of the policies $\{ \pi_1, \pi_2, ... \pi_k \}$ to be Markovian (\textit{i.e.}, it requires the algorithm to never perform policy updates mid-episode), whereas our Theorem \ref{thm:main} does not.

\subsection{Speculative impact}
\label{sec:speculative}
Less directly, many RL domains or algorithms may benefit from Theorem \ref{thm:main} as they rely on the use of non-Markovian policies (generally a collection of Markovian policies). In most of these domains, the non-Markovian property of the policies is a feature, not an issue: it allows one to break down and better compound some conflicting objectives, to induce diversity, and/or to design new policies from elementary ones. Theorem \ref{thm:main} may be a powerful tool for their respective convergence guarantees by proving that their non-Markovian policy admits a well-studied (i.e., Markovian) policy emulating its occupancy measure.

\textbf{Non-Markovian policy induced by the problem setting:} Behavioral Cloning \cite{Urbancic1994,Torabi2018} and Imitation Learning~\cite{Ross2011,Ho2016,Hussein2017} consist in training an agent to reproduce an expert behavior from demonstrations. Sometimes, the expert behavior collection is generated from several near-optimal policies. Moreover, some approaches involve interactive data collection processes in order to make sure that the agent can recover from its own errors. In both cases, similarly to Offline RL, the collected data does not come from a single Markovian policy and Theorem \ref{thm:main} may prove to be useful.
    
\textbf{Non-Markovian policy induced by algorithmic family:} Multi-objective algorithms often combine several policies to generate a behavior that matches the new objective trade-offs~\cite{shelton2001importance,vamplew2009constructing}. Ensemble RL~\cite{Wiering2008}, algorithm selection for RL~\cite{Laroche2018AS}, diversity-induced exploration~\cite{eysenbach2018diversity}, and curriculum for RL~\cite{czarnecki2018mix} all rely on training a family of RL agents. Some more theoretical papers focus on non-Markovian policies~\cite{wu2004non,Scherrer2012}. More and more, policy gradient algorithms utilize and maintain several policies. The PC-PG algorithm~\cite{Agarwal2020} consists in improving the global convergence guarantees of policy gradient methods by implementing an initial state distribution that covers the whole state space. To do so, they learn a policy cover that is made of multiple Markovian policies. Jekyll \& Hyde~\cite{Laroche2021} is another actor-critic algorithm improving convergence guarantees by maintaining two Markovian policies: one dedicated to pure exploration and the other to pure exploitation. Finally, it is worth mentioning distributed agents which perform training updates over several behavioral policies \cite{mnih2016asynchronous,Horgan2018,Schmitt2020}.

\section{Proof of Theorem \ref{thm:main}}
\label{sec:proofs}
\begin{proof}[Proof of Theorem \ref{thm:main}]
Let us start with the first bullet point, concerning the existence of $\tilde{\pi}$ and the fact that it is indeed a Markovian policy.

Letting $\alpha\subseteq\mathcal{A}$ be a measurable set in $\mathcal{A}$. We see that for any $\sigma\subseteq\mathcal{S}$ measurable, we have $\mu^\pi_{\gamma}(\sigma, \alpha) \leq \mu^\pi_{\gamma}(\sigma)$. This directly implies that $\mu^\pi_{\gamma}(\cdot, \alpha)$ is absolutely continuous with respect to $\mu^\pi_{\gamma}(\cdot)$ (both seen as measures on $\mathcal{S}$). Since $\mu^\pi_{\gamma}(\cdot)$ is $\sigma$-finite, the Radon-Nikodym theorem states that $\mu^\pi_{\gamma}(\cdot, \alpha)$ admits a density with respect to $\mu^\pi_{\gamma}(\cdot)$, we let $\tilde{\pi}$ denote that Radon-Nikodym derivative.

We now prove that $\tilde{\pi}$ is a probability measure. The non-negativity is directly inherited from that of measure $\mu^\pi_{\gamma}$. The null-empty set comes from the fact that $\mu^\pi_{\gamma}(\sigma, \emptyset)=0$ for all measurable sets $\sigma$. The countable additivity is a consequence of the countable additivity of the Radon Nykodym derivative and of the measure $\mu^\pi_{\gamma}$. And finally, it is clear from its definition that $\tilde{\pi}(\mathcal{A} | s)=1$ and that $\tilde{\pi}$ only depends on the current state, which makes it a Markovian policy.

Let us now move the second bullet point, which is our core result. Recalling that $\mu^{\pi}_{\gamma}(ds) = \int_{\mathcal{A}} \mu_{\gamma}^{\pi}(ds,da)$, and $\mu^{\tilde{\pi}}_{\gamma}(ds) = \int_{\mathcal{A}} \mu_{\gamma}^{\tilde{\pi}}(ds,da)$, we wish to prove that for all $\sigma \in \Sigma_\mathcal{S}$: $\mu^{\pi}_{\gamma}(\sigma) = \mu^{\tilde{\pi}}_{\gamma}(\sigma)$. Letting $\sigma \in \Sigma_\mathcal{S}$ denote a measurable set such that $\mu^{\pi}_{\gamma}(\sigma) < \infty$, we know from the conservation of mass property of $\sigma$-finite occupancy measures (see Proposition~\ref{prop:conservation} in Appendix~\ref{app:auxilliary}) that:
\begin{align}
    \mu^{\pi}_{\gamma}(\sigma) &= p_0(\sigma) + \gamma \int_{\mathcal{S}} \int_{\mathcal{A}} \mu^{\pi}_{\gamma}(ds_{-1}, da) \; p(\sigma | s_{-1}, a), 
\end{align}
where $p(\sigma | s_{-1}, a)$ denotes the probability of transitioning to $\sigma$ when taking action $a$ in state $s_{-1}$. Now, by definition of the Radon-Nikodym derivative, we have $\mu^{\pi}_{\gamma}(ds, da) = \mu^{\pi}_{\gamma}(ds) \frac{d \mu^\pi_{\gamma}(\cdot \times da)}{d \mu^\pi_{\gamma}(\cdot \times \mathcal{A})}(s) = \mu^{\pi}_{\gamma}(ds) \tilde{\pi}(da|s)$. Using that property, we see that:
\begin{align}
    \mu^{\pi}_{\gamma}(\sigma) &= p_0(\sigma) + \gamma \int_{\mathcal{S}} \int_{\mathcal{A}} \mu^{\pi}_{\gamma}(ds_{-1}) \tilde{\pi}(da|s_{-1}) \; p(\sigma | s_{-1}, a) \\
    &= p_0(\sigma) + \gamma \int_{\mathcal{S}} \mu^{\pi}_{\gamma}(ds_{-1}) \; p^{\tilde{\pi}}(s_{-1}, \sigma),
\end{align}
where $p^{\tilde{\pi}}(s_{-1}, ds)$ is the Markov kernel on $\mathcal{S} \times \mathcal{S}$ obtained by composition of $p$ and $\tilde{\pi}$, and $p^{\tilde{\pi}}(s_{-1}, \sigma)$ is the probability of transitioning to $\sigma$ when acting according to $\tilde{\pi}$ in state $s_{-1}$. Applying the above conservation equality recursively $t$ times gives:
\begin{align}
    \mu^{\pi}_{\gamma}(\sigma) = p_0(\sigma) &+ \gamma \int_{\mathcal{S}} p_0(ds_{-1}) p^{\tilde{\pi}}(s_{-1}, \sigma) + \cdots + \gamma^t \int_{\mathcal{S}} p_0(ds_{-t}) \; p^{\tilde{\pi}}_t(s_{-t}, \sigma) \nonumber \\
   &+ \gamma^{t+1} \int_{\mathcal{S}} \mu^{\pi}_{\gamma}(ds_{-t-1}) \; p^{\tilde{\pi}}_{t+1}(s_{-t-1}, \sigma),
   \label{eq:recursion}
\end{align}
where $p^{\tilde{\pi}}_t$ denotes the composition of $p^{\tilde{\pi}}$ with itself $t$ times. The equality can easily be shown by induction using Fubini's theorem. Given the finiteness of $\mu^{\pi}_{\gamma}(\sigma)$ and the positivity of all the terms involved, we infer there exists $a \geq 0$ such that $\gamma^{t+1} \int_{\mathcal{S}} \mu^{\pi}_{\gamma}(ds_{-t-1}) \; p^{\tilde{\pi}}_{t+1}(s_{-t-1}, \sigma) \to_{t \to \infty} a$. Assuming $a = 0$ (which we shall prove later), we obtain:
\begin{align}
    \mu^{\pi}_{\gamma}(\sigma) &= p_0(\sigma) + \sum_{t=1}^{\infty} \gamma^t \int_{\mathcal{S}} p_0(ds_{-t}) \; p^{\tilde{\pi}}_t(s_{-t}, \sigma).
\end{align}
Now, by the very definition of occupancy measures and Markov policies, we see that $p_0(\sigma) + \sum_{t=1}^{T} \gamma^t \int_{\mathcal{S}} p_0(ds_{-t}) \; p^{\tilde{\pi}}_t(s_{-t}, \sigma)$ is the partial sum of $\mu^{\tilde{\pi}}_{\gamma}(\sigma)$ in Equation~\ref{eq:occupancy-measure}:\footnote{Note that in the general case, it is not the partial sum of $\mu^{\pi}_{\gamma}(\sigma)$ due to $\pi$'s non-Markovian character.}
\begin{align}
    \mu^{\tilde{\pi}}_{\gamma}(\sigma,\mathcal{A}) = \lim_{T \to +\infty}\mathbb{E}\left[\sum_{t=0}^T \gamma^t\mathds{1}\left(S_t \in \sigma\right)\times\mathds{1}\left(A_t \in \mathcal{A}\right) \bigg|\!\!\begin{array}{l}
S_0\sim p_0(\cdot), A_t \sim \tilde{\pi}(\cdot|S_{t}),\\S_{t+1}\sim p(\cdot|S_t,A_t)\end{array}\right] .
\label{eq:partial}
\end{align}
The convergence of this partial sum is guaranteed by the finiteness of $\mu^{\pi}_{\gamma}(\sigma)$. In other words, we get:
\begin{align}
    \mu^{\tilde{\pi}}_{\gamma}(\sigma) = p_0(\sigma) + \sum_{t=1}^{\infty} \gamma^t \int_{\mathcal{S}} p_0(ds_{-t}) \; p^{\tilde{\pi}}_t(s_{-t}, \sigma) = \mu^{\pi}_{\gamma}(\sigma) < +\infty.
    \label{eq:equality}
\end{align}

Now, for any arbitrary $\sigma \in \Sigma_\mathcal{S}$ (possibly of infinite $\mu^{\pi}_{\gamma}$-measure), we know from the $\sigma$-finiteness of $\mu^{\pi}_{\gamma}$ that there exists a sequence $(\sigma_n)_{n \in \mathbb{N}}$ of disjoint measurable sets such that $\forall n \in \mathbb{N}, \mu^{\pi}_{\gamma}(\sigma_n) < +\infty$ and $\sigma = \cup_{n=0}^{\infty} \sigma_n$. We compute:
\begin{align}
    \mu^{\pi}_{\gamma}(\sigma) = \sum_{n=0}^{\infty} \mu^{\pi}_{\gamma}(\sigma_n) = \sum_{n=0}^{\infty} \mu^{\tilde{\pi}}_{\gamma}(\sigma_n) = \mu^{\tilde{\pi}}_{\gamma}(\sigma).
\end{align}
Combining this equality with the policy definition \ref{eq:ts2stoch} gives the final result: $\mu^{\pi}_{\gamma}(ds, da) = \mu^{\tilde{\pi}}_{\gamma}(ds, da)$.

We are left with proving that $\lim_{t \to \infty} \gamma^{t+1} \int_{\mathcal{S}} \mu^{\pi}_{\gamma}(ds_{-t-1}) \; p^{\tilde{\pi}}_{t+1}(s_{-t-1}, \sigma) = 0$. It is obvious when $\gamma < 1$, since in that case, $\int_{\mathcal{S}} \mu^{\pi}_{\gamma}(ds_{-t-1}) < \infty$, implying that the integral term is bounded.

The case $\gamma = 1$ is somewhat more involved (as is customary with undiscounted MDPs). We start by noticing that
$\int_{\mathcal{S}} \mu^{\tilde{\pi}}_{\gamma}(ds_{-t-1}) \; p^{\tilde{\pi}}_{t+1}(s_{-t-1}, \sigma) \to_{t \to \infty} 0$. This stems directly from applying Eq.~\eqref{eq:recursion} to $\tilde{\pi}$ and then leveraging the first equality in~\eqref{eq:equality} .
Now, we let $L = \{ s \in \mathcal{S} \;|\; p^{\tilde{\pi}}_{t}(s, \sigma) \to_{t \to \infty} 0 \}$. From $\int_{L^C} \mu^{\tilde{\pi}}_{\gamma}(ds_{-t-1}) \; p^{\tilde{\pi}}_{t+1}(s_{-t-1}, \sigma) \to_{t \to \infty} 0$ and by the definition of $L$, we infer that $\mu^{\tilde{\pi}}_{\gamma}(L^C) = 0$. It also stems directly from Proposition~\ref{prop:ac-traj} that $\mu^{\pi}_{\gamma}$ is absolutely continuous with respect to $\mu^{\tilde{\pi}}_{\gamma}$, which implies that $\mu^{\pi}_{\gamma}(L^C) = 0$. Finally, by the dominated convergence theorem (applicable only if $\mu^{\pi}_{\gamma}(\mathcal{S}) < +\infty$, see Appendix \ref{app:dominated} for the proof in the most general case):
\begin{align}
    \int_{\mathcal{S}} \mu^{\pi}_{\gamma}(ds_{-t-1}) \; p^{\tilde{\pi}}_{t+1}(s_{-t-1}, \sigma) = \int_{L} \mu^{\pi}_{\gamma}(ds_{-t-1}) \; p^{\tilde{\pi}}_{t+1}(s_{-t-1}, \sigma) \to_{t \to \infty} 0,
\end{align}
which concludes the proof.
\end{proof}

\section{Conclusion}
\label{sec:conclusion}
In this paper, we developed a general theory of the occupancy measure in MDPs and proved that, for any non-Markovian policy admitting a $\sigma$-finite occupancy, there exists a Markovian policy with the same occupancy. We illustrate the impact of this theory by extending existing proofs in the literature to non-Markovian policies.

\newpage
\bibliographystyle{abbrv}
\bibliography{ref}

\newpage
\section*{Checklist}


\begin{enumerate}

\item For all authors...
\begin{enumerate}
  \item Do the main claims made in the abstract and introduction accurately reflect the paper's contributions and scope?
    \answerYes{}
  \item Did you describe the limitations of your work?
    \answerYes{See Section \ref{sec:theory}}
  \item Did you discuss any potential negative societal impacts of your work?
    \answerNA{This is a theoretical work.}
  \item Have you read the ethics review guidelines and ensured that your paper conforms to them?
    \answerYes{}
\end{enumerate}

\item If you are including theoretical results...
\begin{enumerate}
  \item Did you state the full set of assumptions of all theoretical results?
    \answerYes{For concision, given the high universality of our main theorem, and the technicality of the condition, the $\sigma$-finiteness condition is enunciated neither in the abstract, nor the first part of the introduction, but clearly stated in the formal theorem.}
	\item Did you include complete proofs of all theoretical results?
    \answerYes{The main theorems are proved in the main document and the rest in the supplementary material.}
\end{enumerate}

\item If you ran experiments...
    \answerNA{No experiment}

\item If you are using existing assets (e.g., code, data, models) or curating/releasing new assets...
    \answerNA{No asset has been used nor released}

\item If you used crowdsourcing or conducted research with human subjects... \answerNA{No crowdsourcing or humans involved else than the workforce of the authors.}
\end{enumerate}

\newpage
\appendix
\include{appendix-theory}

\end{document}

%% file: appendix-theory.tex



\section{Theory}
\label{app:theory}

\subsection{Proof of Theorem~\ref{thm:main}}
\label{app:dominated}

We complete the proof of our main theorem here.

\begin{proof}
    For the dominated convergence theorem to be applicable, one needs for instance $\mu^{\pi}_{\gamma}(\mathcal{S}) < +\infty$ (in which case $|p^{\tilde{\pi}}_{t+1}(s_{-t-1}, \sigma)| \leq 1, \forall s_{-t-1} \in \mathcal{S}$ and $1$ is integrable under $\mu^{\pi}_{\gamma}(\mathcal{S})$). We have thus established so far that $\mu^{\pi}_{\gamma}(\sigma) = \mu^{\tilde{\pi}}_{\gamma}(\sigma)$ in the case of a finite $\mu^{\pi}_{\gamma}(\mathcal{S})$.
    
    Let us now show that it also holds when $\mu^{\pi}_{\gamma}(\mathcal{S}) = \infty$ (still assuming $\mu^{\pi}_{\gamma}$ is $\sigma$-finite). To that end, we are going to create an auxiliary MDP on which $\pi$ and $\tilde{\pi}$ behave approximately as in the initial MDP in a large set of finite $\mu^{\pi}_{\gamma}$-measure ``around'' $\sigma$, but where both have finite occupancy measures.
    
    We denote by $\mu^{\tilde{\pi},T}_{\gamma}(\sigma)$ the partial sum of $\mu^{\tilde{\pi}}_{\gamma}(\sigma)$ (\textit{i.e.}, the sum till $T$ in the definition of occupancy measures, \textit{e.g.}, equation~\eqref{eq:partial}). Similarly, $\mu^{\pi,T}_{\gamma}(\sigma)$ denotes the partial sum of $\mu^{\pi}_{\gamma}(\sigma)$.
    
    Let $\epsilon > 0$. Given the finiteness of $\mu^{\pi}_{\gamma}(\sigma)$ and $\mu^{\tilde{\pi}}_{\gamma}(\sigma)$, there exists $T > 0$ such that:
    \begin{align}
        \mu^{\tilde{\pi}}_{\gamma}(\sigma) - \mu^{\tilde{\pi},T}_{\gamma}(\sigma) \leq \epsilon \hspace{1cm} \text{and} \hspace{1cm} \mu^{\pi}_{\gamma}(\sigma) - \mu^{\pi,T}_{\gamma}(\sigma) \leq \epsilon.
    \end{align}
    The first important point is that $\mu^{\pi,T}_{\gamma}(\mathcal{S}) \leq T + 1 < +\infty$.
    
    Now, let $(\sigma_i)_{i \in \mathbb{N}}$ be a set of disjoint measurable sets covering $\mathcal{S}$ such that $\sigma_0 = \sigma$ and $\mu^{\pi}_{\gamma}(\sigma_i) < \infty$ for all $i \in \mathbb{N}$ (such a construction clearly exists by the $\sigma$-finiteness of $\mu^{\pi}_{\gamma}$). From the finiteness of $\mu^{\pi,T}_{\gamma}(\mathcal{S})$, there exists $N \in \mathbb{N}$ such that:
    \begin{align}
        \mu^{\pi,T}_{\gamma}(\mathcal{S}) - \mu^{\pi,T}_{\gamma}(\bigcup_{0 \leq i \leq N} \sigma_i) \leq \frac{\epsilon}{T}.
        \label{eq:s_epsilon}
    \end{align}
    We let $\mathcal{S}^\epsilon \defeq \bigcup_{0 \leq i \leq N} \sigma_i$, and emphasize that $\mu^{\pi}_{\gamma}(\mathcal{S}^\epsilon) < \infty$\footnote{This is the part of the proof that breaks down when the measure $\mu^{\pi}_{\gamma}$ is not $\sigma$-finite.}.
    
    We then define the auxiliary MDP $m^\epsilon=\langle \mathcal{S}, \mathcal{A},p_0^\epsilon,p^\epsilon,r,\gamma \rangle$, where $p_0^\epsilon(\xi) = p_0(\xi \cap \mathcal{S}^\epsilon)$\footnote{Note that $p_0(\mathcal{S}) \leq 1$ in this new construction.} and $p^\epsilon(\xi | s, a) = p(\xi \cap \mathcal{S}^\epsilon | s, a)$ for any $\xi \in \Sigma_{\mathcal{S}}$. In this new MDP, any transition that exits $\mathcal{S}^\epsilon$ is terminal. The policies $\pi$ and $\tilde{\pi}$ are kept identical. We let $\nu^{\pi}_{\gamma}$ and $\nu^{\tilde{\pi}}_{\gamma}$ denote the occupancy measures of $\pi$ and $\tilde{\pi}$ in $m^\epsilon$. The following properties hold:
    \begin{enumerate}[(i).]
        \item $\nu^{\pi}_{\gamma}(\mathcal{S}) \leq \mu^{\pi}_{\gamma}(\mathcal{S}^\epsilon) < +\infty$,
        \item $\nu^{\pi,T}_{\gamma}(\sigma) \leq \nu^{\pi}_{\gamma}(\sigma) \leq \mu^{\pi}_{\gamma}(\sigma)$,
        \item $0 \leq \mu^{\pi,T}_{\gamma}(\sigma) - \nu^{\pi,T}_{\gamma}(\sigma) \leq \epsilon$,
        \item $0 \leq \nu^{\pi}_{\gamma}(\sigma) - \nu^{\pi,T}_{\gamma}(\sigma) \leq 2\epsilon$,
        \item $0 \leq \nu^{\tilde{\pi}}_{\gamma}(\sigma) - \nu^{\tilde{\pi},T}_{\gamma}(\sigma) \leq 2\epsilon$.
    \end{enumerate}
    (i) and (ii) follow directly from the definition of $m^\epsilon$. (iii) stems from equation \eqref{eq:s_epsilon}: the mass in $\sigma$ lost from exiting $\mathcal{S}^\epsilon$ prior to timestep $T$ is at most $T \times \frac{\epsilon}{T}$.
    For (iv) (and similarly (v)), we have:
    \begin{align}
        \nu^{\pi}_{\gamma}(\sigma) - \nu^{\pi,T}_{\gamma}(\sigma) \leq \mu^{\pi}_{\gamma}(\sigma) - \nu^{\pi,T}_{\gamma}(\sigma) \leq \mu^{\pi}_{\gamma}(\sigma) - \mu^{\pi, T}_{\gamma}(\sigma) + \mu^{\pi,T}_{\gamma}(\sigma) - \nu^{\pi,T}_{\gamma}(\sigma) \leq 2 \epsilon.
    \end{align}
    Importantly, (i) implies that $\mathcal{S}$ has finite measure under $\nu^{\pi}_{\gamma}$, allowing us to apply our result on finite measures, guaranteeing: $\nu^{\pi}_{\gamma}(\sigma) = \nu^{\tilde{\pi}}_{\gamma}(\sigma)$.
    Using the various inequalities above, we get:
    \begin{align}
        \mu^{\pi}_{\gamma}(\sigma) - \nu^{\pi}_{\gamma}(\sigma) = \mu^{\pi}_{\gamma}(\sigma) - \mu^{\pi,T}_{\gamma}(\sigma) + \mu^{\pi,T}_{\gamma}(\sigma) - \nu^{\pi,T}_{\gamma}(\sigma) + \nu^{\pi,T}_{\gamma}(\sigma) - \nu^{\pi}_{\gamma}(\sigma)
        \leq 2 \epsilon,
    \end{align}
    and
    \begin{align}
        \nu^{\tilde{\pi}}_{\gamma}(\sigma) - \mu^{\tilde{\pi}}_{\gamma}(\sigma) &= \nu^{\tilde{\pi}}_{\gamma}(\sigma) - \nu^{\tilde{\pi},T}_{\gamma}(\sigma) + \nu^{\tilde{\pi},T}_{\gamma}(\sigma) - \mu^{\tilde{\pi},T}_{\gamma}(\sigma) +\mu^{\tilde{\pi},T}_{\gamma}(\sigma)  - \mu^{\tilde{\pi}}_{\gamma}(\sigma) \leq 2 \epsilon.
    \end{align}
    Summing those two inequalities gives:
    \begin{align}
        \mu^{\pi}_{\gamma}(\sigma) - \nu^{\pi}_{\gamma}(\sigma) + \nu^{\tilde{\pi}}_{\gamma}(\sigma) - \mu^{\tilde{\pi}}_{\gamma}(\sigma) = \mu^{\pi}_{\gamma}(\sigma) - \mu^{\tilde{\pi}}_{\gamma}(\sigma) \leq 4 \epsilon.
    \end{align}
    Since this holds for any $\epsilon > 0$, we get that $\mu^{\pi}_{\gamma}(\sigma) - \mu^{\tilde{\pi}}_{\gamma}(\sigma) \leq 0$, which concludes the proof (the inequality in the other direction has already been established).
\end{proof}

\subsection{Auxiliary Result}
\label{app:auxilliary}

For the sake of completeness, we include the following very standard result.

\begin{prop}[Conservation of mass]
    For any policy $\pi$ whose occupancy measure is $\sigma$-finite, we have:
    \begin{align}
        \mu^{\pi}_{\gamma}(ds) &= p_0(ds) + \gamma \int_{\mathcal{S}} \int_{\mathcal{A}} \mu^{\pi}_{\gamma}(ds_{-1}, da) \; p(ds|s_{-1}, a).
    \end{align}
    \label{prop:conservation}
\end{prop}
\begin{proof}
    Note that the equality is meant in the sense of measures, that is, for any $\sigma \in \Sigma_{\mathcal{S}}$: 
    \begin{align}
        \mu^{\pi}_{\gamma}(\sigma) &= p_0(\sigma) + \gamma \int_{\mathcal{S}} \int_{\mathcal{A}} \mu^{\pi}_{\gamma}(ds_{-1}, da) \; p(\sigma|s_{-1}, a).
        \label{eq:conservation}
    \end{align}
    Let us start by proving that this equality holds for any $\sigma$ such that $\mu^{\pi}_{\gamma}(\sigma) < +\infty$:
    \begin{align}
        \mu^{\pi}_{\gamma}(\sigma,\alpha) \coloneqq& \mathbb{E}\left[\sum_{t=0}^\infty \gamma^t\mathds{1}\left(S_t \in \sigma\right)\times\mathds{1}\left(A_t \in \alpha\right) \bigg|\!\!\begin{array}{l}
    S_0\sim p_0(\cdot), A_t \sim \pi(\cdot|H_{t}),\\S_{t+1}\sim p(\cdot|S_t,A_t)\end{array}\right] \\
    =& \sum_{t=0}^\infty \gamma^t \int_\sigma \int_\alpha p_t(ds, da),
    \end{align}
    where the second line is a direct application of Fubini's theorem (valid since $\mu^{\pi}_{\gamma}(\sigma) < +\infty$), and $p_t(ds, da)$ denotes the measure on $\mathcal{S} \times \mathcal{A}$ induced by $(S_t,A_t)$ when both evolve according to $p_0$, $\pi$ and $p$. Note in particular that $p_0(ds, da) = p_0(ds) \pi(da|s)$ where $p_0(ds)$ is the initial state distribution in the MDP. Naturally, $p_t(ds)$ denotes the marginal of $p_t(ds, da)$ on $\mathcal{S}$.
    Focusing on $\alpha = \mathcal{A}$, we see that:
    \begin{align}
        \mu^{\pi}_{\gamma}(\sigma) &= \sum_{t=0}^\infty \gamma^t \int_\sigma p_t(ds) = p_0(\sigma) + \sum_{t=1}^\infty \gamma^t \int_\sigma p_t(ds) \\
        &= p_0(\sigma) + \sum_{t=1}^\infty \gamma^t \int_\sigma \int_{\mathcal{S}} \int_\mathcal{A} p_{t-1}(ds_{-1}, da) \; p(ds | s_{-1}, a) \\
        &= p_0(\sigma) + \gamma \int_\sigma \left[ \sum_{t=1}^\infty \gamma^{t-1} \int_{\mathcal{S}} \int_\mathcal{A} p_{t-1}(ds_{-1}, da) \right] \; p(ds | s_{-1}, a) \\
        &= p_0(\sigma) + \gamma \int_\sigma \int_{\mathcal{S}} \int_\mathcal{A} \mu^{\pi}_{\gamma}(ds_{-1},da) \; p(ds | s_{-1}, a),
    \end{align}
    where all integrals on $\sigma$ are with respect to $s$, Fubini was applied again, and we used the equality $p_t(ds) = \int_{\mathcal{S}} \int_\mathcal{A} p_{t-1}(ds_{-1}, da) \; p(ds | s_{-1}, a)$ that stems directly from the definition of a Markov Decision Process.
    
    Finally, for any $\sigma \in \Sigma_{\mathcal{S}}$, we know from the $\sigma$-finiteness of $\mu^{\pi}_{\gamma}$ that there exists a sequence $(\sigma_n)_{n \in \mathbb{N}}$ of disjoint measurable sets such that $\forall n \in \mathbb{N}, \mu^{\pi}_{\gamma}(\sigma_n) < +\infty$ and $\sigma = \cup_{n=0}^{\infty} \sigma_n$. Applying Eq.~\ref{eq:conservation} to $\sigma_n$ and summing over $n \in \mathbb{N}$ concludes the proof.
\end{proof}

\subsection{Proof of Theorem \ref{thm:measure}}
\label{app:occupancymeasure}
\measure*
\begin{proof}[Proof of Theorem \ref{thm:measure}]
    Fixing $\sigma\in\Sigma_\mathcal{S}$ and $\alpha\in\Sigma_\mathcal{A}$, we notice that the following sequence is increasing with $T$:
    \begin{align}
        U_T \coloneqq& \mathbb{E}\left[\sum_{t=0}^T \gamma^t\mathds{1}\left(S_t \in \sigma\right) \times\mathds{1}\left(A_t \in \alpha\right) \bigg|\!\!\begin{array}{l}
    S_0\sim p_0(\cdot), A_t \sim \pi(\cdot|H_{t}),\\S_{t+1}\sim p(\cdot|S_t,A_t)\end{array}\right] .\nonumber
    \end{align}
    Therefore, the monotone convergence theorem guarantees that $U_T$ converges on $\mathbb{R}^{\myplus}\cup\{+\infty\}$ when $T$ tends to infinity and its limit is by construction the occupancy measure of $\pi$, which is therefore defined for every state set $\sigma$ and action set $\alpha$.
    
    In order to establish that $\mu_{\gamma}^\pi$ is a measure over the algebra product $\Sigma_\mathcal{S}\times\Sigma_\mathcal{A}$, we need to check (i) its positivity, (ii) that $\mu_{\gamma}^\pi(\emptyset)=0$, and (iii) its countable additivity with respect to disjoint sets. (i) has been established right before, (ii) is a direct consequence that $\mathds{1}\left(S_t \in \emptyset\right)\times\mathds{1}\left(A_t \in \emptyset\right)=0$, and (iii) is a simple summation order change, justified by the positivity of all quantities involved.
\end{proof}

\subsection{Characterization of performance}
\label{app:performanceequivalence}

\performance*
\begin{proof}
    We recall the definition of the performance $\rho^{\pi}_{\gamma}$:
        \begin{align}
            \rho^{\pi}_{\gamma} \coloneqq& \mathbb{E}\left[\sum_{t=0}^\infty \gamma^t R_t \bigg|\!\!\begin{array}{l}
        S_0\sim p_0(\cdot), A_t \sim \pi(\cdot|H_{t}),\\R_t\sim r(S_t,A_t),S_{t+1}\sim p(\cdot|S_t,A_t)\end{array}\right].
        \end{align}
    We start by noting that the existence of $\rho^{\pi}_{\gamma}$ is not guaranteed in the case of $\gamma = 1$ and $\mu_\gamma^\pi(\mathcal{S}) = +\infty$. Assuming it does exist (or that any of those two conditions is not verified), the law of total expectation gives:
        \begin{align}
            \rho^{\pi}_{\gamma} \coloneqq& \mathbb{E}\left[\sum_{t=0}^\infty \gamma^t \mathbb{E}[r(S_t, A_t)] \bigg|\!\!\begin{array}{l}
        S_0\sim p_0(\cdot), A_t \sim \pi(\cdot|H_{t}),\\S_{t+1}\sim p(\cdot|S_t,A_t)\end{array}\right].
        \end{align}
    Applying the law of total expectation a second time, and Fubini-Tonelli's theorem allows to conclude.
\end{proof}

\perfequiv*
\begin{proof}
    Since $\mu_\gamma^\pi = \mu_\gamma^{\tilde{\pi}}$, we get from Lemma~\ref{lem:performance}: $$\rho^{\pi}_{\gamma} =  \int_{\mathcal{S}} \int_{\mathcal{A}} \mathbb{E}\left[r(s,a)\right]\mu_\gamma^\pi(ds,da) =  \int_{\mathcal{S}} \int_{\mathcal{A}} \mathbb{E}\left[r(s,a)\right]\mu_\gamma^{\tilde{\pi}}(ds,da) = \rho^{\tilde{\pi}}_{\gamma},$$
    which concludes the proof.
\end{proof}

See Section~\ref{app:infinite_occupancy_cex} below for the proof of Proposition~\ref{prop:sigma-infinite}.

\subsection{Counter-examples under $\sigma$-infinite measures}
\label{app:infinite_occupancy_cex}

\sigmainfinite*

We detail in the next two sections, two counter-examples proving the proposition.

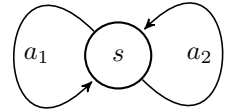
\begin{wrapfigure}{r}{0.35\textwidth}
        \vspace{-40pt}
        \begin{center}
            \scalebox{1}{
                \begin{tikzpicture}[->, >=stealth', scale=1.5 , semithick, node distance=2cm]
                    \tikzstyle{every state}=[fill=white,draw=black,thick,text=black,scale=1]
                    \node[state]    (x0)                {$s$};
                    \path
                    (x0) edge[out=315, in=45, loop, left]    node{$a_2$}     (x0)
                    (x0) edge[out=135, in=225, loop, right]    node{$a_1$}     (x0);
                \end{tikzpicture}
            }
            \vspace{-20pt}
            \caption{Minimal MDP such that $p(s|s,a_1)=1$ and $p(s|s,a_2)=1$.}
            \label{fig:minimalMDP2}
        \end{center}
        \vspace{-20pt}
\end{wrapfigure}
\paragraph{Example where $\tilde{\pi}$ is undetermined}
We use the minimal undiscounted ($\gamma=1$) MDP in figure \ref{fig:minimalMDP2} with a single state $\mathcal{S} = \{s\}$ and two actions $\mathcal{A}= \{a_1,a_2\}$ that loop: $p(s|s,a_1)=1$ and $p(s|s,a_2)=1$. We consider the non-Markovian policy $\pi$ that chooses its action as follows: $\forall i\in\mathbb{N}$, if $t\in[2^i,2^{i+1})$, then $\pi(a_1)=\frac{(-1)^i + 1}{2}$, \textit{i.e.} during each epoch $i$, the policy plays deterministically $a_1$ if $i$ is even otherwise $a_2$ if $i$ is odd.

Then, the ratio $\tilde{\pi}\doteq \frac{\mu_{1}^{\pi}(s,a_1)}{\mu_{1}^{\pi}(s)}$ is undetermined as the limit of the action selection ratio does not converge as $t$ goes to infinity. We could argue that any non-determinisitic Markovian policy would admit the same occupancy measure as $\pi$: $\mu_{1}^{\pi}(s,a_1)=\mu_{1}^{\pi}(s,a_2)=\infty$, but the counter-example in the next subsection shows that this is not always possible.

\paragraph{Example where no Markovian policy reproduces the occupancy of a non-Markovian policy}
We use the minimal undiscounted ($\gamma=1$) MDP in figure \ref{fig:minimalMDP} with a single state $\mathcal{S} = \{s\}$ and two actions $\mathcal{A}= \{a_1,a_2\}$, such that $p(s|s,a_1)=1$ and $a_2$ is terminal. We consider the non-Markovian policy $\pi$ that chooses its action uniformly at timestep $t=0$ and deterministically $a_1$ for $t\geq 1$. Its occupancy measure is therefore:
\begin{align}
     \mu_{1}^{\pi}(s,a_1) &= \infty & \mu_{1}^{\pi}(s,a_2) &= \frac{1}{2}.
\end{align}

The set of Markovian policies $\pi_\theta\in\Pi_\textsc{m}$ may be parametrized with a single parameter $\pi_\theta(a_1|s)\doteq\theta\in[0,1]$ and $\pi_\theta(a_2|s)\doteq 1-\theta\in[0,1]$. A Markovian policy $\pi_\theta$ admits the following occupancy measure:
\begin{align}
     \text{if }\theta<1, \quad\quad \mu_{1}^{\pi_\theta}(s,a_1) &= \frac{1}{1-\theta} & \mu_{1}^{\pi_\theta}(s,a_2) &= 1,\\
     \text{if }\theta=1, \quad\quad \mu_{1}^{\pi_\theta}(s,a_1) &= \infty & \mu_{1}^{\pi_\theta}(s,a_2) &= 0,
\end{align}
none of which match the occupancy of $\pi$.

\subsection{Idempotence and projection}
\label{app:idempotence}
\idempotence*
\begin{proof}
    Given the definition of $\tilde{\pi}(\alpha|s) = \frac{d \mu^\pi_{\gamma}(\cdot, \alpha)}{d \mu^\pi_{\gamma}(\cdot ,  \mathcal{A})}(s)$, to show that $\tilde{\pi}=\pi$, we simply need to prove that for any $\sigma \in \Sigma_{\mathcal{S}}$ and $\alpha \in \Sigma_{\mathcal{A}}$:
    \begin{align}
        \mu^\pi_{\gamma}(\sigma, \alpha) = \int_\sigma \mu^\pi_{\gamma}(ds, \mathcal{A}) \pi(\alpha|s).
    \end{align}
    From the definition~\ref{eq:occupancy-measure} of $\mu^\pi_{\gamma}(\sigma, \alpha)$, we have:
    \begin{align}
        \mu^{\pi}_{\gamma}(\sigma,\alpha) =& \mathbb{E}\left[\sum_{t=0}^\infty \gamma^t\mathds{1}\left(S_t \in \sigma\right)\times\mathds{1}\left(A_t \in \alpha\right) \bigg|\!\!\begin{array}{l}
    S_0\sim p_0(\cdot), A_t \sim \pi(\cdot|H_{t}),\\S_{t+1}\sim p(\cdot|S_t,A_t)\end{array}\right] \\
    =& \sum_{t=0}^\infty \gamma^t \int_\sigma \int_\alpha p_t(ds, da) \label{eq:sum}
    \end{align}
    where we reused the notations from the proof of Proposition~\ref{prop:conservation}. Given the Markovian nature of $\pi$, we see that: $p_t(ds, da) = p_t(ds) \pi(da|s)$. Reintroducing this into~\ref{eq:sum}, we get:
    \begin{align}
        \mu^{\pi}_{\gamma}(\sigma,\alpha) =& \sum_{t=0}^\infty \gamma^t \int_\sigma \int_\alpha p_t(ds) \pi(da|s) = \int_\sigma \sum_{t=0}^\infty \gamma^t p_t(ds) \pi(\alpha|s) = \int_\sigma \mu^{\pi}_{\gamma}(ds) \pi(\alpha|s),
    \end{align}
    which proves that $\pi$ is indeed the Radon-Nikodym derivative $\frac{d \mu^\pi_{\gamma}(\cdot, \alpha)}{d \mu^\pi_{\gamma}(\cdot ,  \mathcal{A})}(s)$ and shows the equality to $\tilde{\pi}$ up to a $\mu^\pi_{\gamma}(\cdot, \mathcal{A})$-null set.
\end{proof}

\subsection{Absolute continuity of finite trajectory distribution}
\label{app:absolutelycontinuous}
\actraj*
\begin{proof}
We prove this result by induction on $t \geq 0$. 
For $t = 0$, let us consider $(\sigma, \alpha) \in \Sigma_\mathcal{S} \times \Sigma_\mathcal{A}$ such that $\tau^{\tilde{\pi}}_{0}(\sigma, \alpha) = 0$.
This implies that:
\begin{align}
    \int_\sigma \int_\alpha p_0(ds) \tilde{\pi}(da | s) = \int_\sigma p_0(ds) \frac{d \mu^\pi_{\gamma}(\cdot, \alpha)}{d \mu^\pi_{\gamma}(\cdot )}(s) = 0.
\end{align}
Letting $N_\alpha = \{ s \in \mathcal{S} \;|\; \frac{d \mu^\pi_{\gamma}(\cdot, \alpha)}{d \mu^\pi_{\gamma}(\cdot )}(s) = 0 \}$, we see that necessarily $p_0(N_\alpha \cap \sigma) = p_0(\sigma)$ (otherwise the above integrals would not be 0). Now, by definition of the Radon-Nikodym derivative, we have:
\begin{align}
    \mu^\pi_{\gamma}(N_\alpha, \alpha) = \int_{N_\alpha} \mu^\pi_{\gamma}(ds) \frac{d \mu^\pi_{\gamma}(\cdot, \alpha)}{d \mu^\pi_{\gamma}(\cdot )}(s) = 0.
\end{align}
Since $p_0(N_\alpha \cap \sigma) = p_0(\sigma)$, we know that: $\tau_0^\pi(\sigma, \alpha) = \tau_0^\pi(N_\alpha \cap \sigma, \alpha) \leq \mu^\pi_{\gamma}(N_\alpha, \alpha) = 0$, which concludes the base case.

Let us proceed to the induction step. We consider $(\sigma, \alpha) \in (\Sigma_\mathcal{S} \times \Sigma_\mathcal{A})^{t+1}$, with $\tau^{\tilde{\pi}}_{t+1}(\sigma, \alpha) = 0$, and aim to prove that $\tau^{\pi}_{t+1}(\sigma, \alpha) = 0$. We let $\sigma_{|t}$ and $\alpha_{|t}$ denote the first $t$ components of $\sigma$ and $\alpha$, and $\sigma_{t+1}$ and $\alpha_{t+1}$ their $t+1$-th. From the Markov property of the various objects involved, we have:
\begin{align}
    \tau_{t+1}^{\tilde{\pi}}(\sigma, \alpha) = \int_{\alpha_{t+1}} \int_{\sigma_{t+1}} \int_{\sigma_{|t}, \alpha_{|t}} \tilde{\pi}(da_{t+1} \;|\; s_{t+1}) p(ds_{t+1} \;|\; s_{t},a_{t}) \tau_{t}^{\tilde{\pi}}(ds_{|t}, da_{|t}). 
\end{align}
As far as $\pi$ is concerned, we have:
\begin{align}
    \tau_{t+1}^{\pi}(\sigma, \alpha) = \int_{\alpha_{t+1}} \int_{\sigma_{t+1}} \int_{\sigma_{|t}, \alpha_{|t}} \pi(da_{t+1} \;|\; s_{|t+1}, a_{|t}) p(ds_{t+1} \;|\; s_{t},a_{t}) \tau_{t}^{\pi}(ds_{|t}, da_{|t}). 
\end{align}
Since $\tau_{t+1}^{\tilde{\pi}}(\sigma, \alpha) = 0$, three cases are possible (corresponding to the measure of the three integrals from right to left being null): 
\begin{enumerate}[(i).]
    \item $\tau_{t}^{\tilde{\pi}}(\sigma_{|t}, \alpha_{|t}) = 0$. In this case, the induction hypothesis implies $\tau_{t}^{\pi}(\sigma_{|t}, \alpha_{|t}) = 0$, and thus $\tau_{t+1}^{\pi}(\sigma, \alpha) = 0$.
    \item $\int_{\sigma_{|t}, \alpha_{|t}} p(\sigma_{t+1} \;|\; s_{t},a_{t}) \tau_{t}^{\tilde{\pi}}(ds_{|t}, da_{|t}) = 0$. We define $N_1 = \{ (s,a) \;|\; p(\sigma_{t+1} \;|\; s, a) \neq 0\}$, and see that necessarily $\tau_{t}^{\tilde{\pi}}(N_1) = 0$, implying that $\tau_{t}^{\pi}(N_1) = 0$ by the induction hypothesis, and thus that $\tau_{t+1}^{\pi}(\sigma, \alpha) = 0$.
    \item $\int_{\sigma_{t+1}} \tau_{t+1}^{\tilde{\pi}}(ds_{t+1}) \tilde{\pi}(\alpha_{t+1}|s_{t+1}) = 0$, where we overloaded the notations by letting $\tau_{t+1}^{\tilde{\pi}}(ds_{t+1})$ be the distribution of the $t+1$-th state in the trajectory when following $\tilde{\pi}$. Similarly as above, this implies that: $\tau_{t+1}^{\tilde{\pi}}(N_\alpha \cap \sigma_{t+1}) = \tau_{t+1}^{\tilde{\pi}}(\sigma_{t+1})$. In addition, using the same argument as in (ii), the induction hypothesis can be applied to get $\tau_{t+1}^{\pi}(\sigma_{|t} \times (N_\alpha \cap \sigma_{t+1}), \alpha) = \tau_{t+1}^{\pi}(\sigma, \alpha)$. Finally, $\tau_{t+1}^{\pi}(\sigma_{|t} \times (N_\alpha  \cap \sigma_{t+1}), \alpha) \leq \mu^\pi_{\gamma}(N_\alpha, \alpha) = 0$.
\end{enumerate}
This concludes the induction step, and with it the proof of the proposition.
\end{proof}

\subsection{Finiteness and $\sigma$-finiteness}
\label{app:finiteness}

The ($\sigma$-)finiteness of the occupancy measure is an interesting property as it will be required to avoid encountering indeterminate formulas. 
We note that $\gamma<1$ suffices to guarantee that all Markovian policies have a finite occupancy measure for any policy. Below, we derive a more general characterisation of MDPs that admit only policies with finite occupancy measures.
\finiteness*
\begin{proof}
    We notice that the expected performance under reward $r(s,a) = 1 \forall s,a$ is equal to the occupancy measure $\mu_{\gamma}^\pi(\mathcal{S})$ over the full state-action pair set. We use the well known theoretical result~\cite{Sutton1998} that there exists a deterministic Markovian policy that optimises any MDP. If the occupancy measure is finite for all deterministic Markovian policies then, we have for all policy $\pi$:
    \begin{align}
        \mu_{\gamma}^\pi(\mathcal{S})\leq \max_{\pi_{dm}\in\Pi_\textsc{dm}} \mu_{\gamma}^{\pi_{dm}}(\mathcal{S}) <\infty,
    \end{align}
    which establishes the finiteness of the occupancy measure of $\pi$.
\end{proof}

However, it is interesting to notice that even if all deterministic Markovian policies have $\sigma$-finite occupancy measures, a non-Markovian policy $\pi$ may admit a $\sigma$-infinite occupancy measure. We construct below such a counter-example.

We consider the deterministic continuous MDP $m$ where $\mathcal{S}=[0,2]$, $s_0=0$, $\mathcal{A}=(0,1]$, $p(s+a|s,a)=1$, $\gamma=1$, and the trajectory terminates when $s+a>2$. 

We start by establishing that any deterministic policy $\pi_{d}$ (Markovian or non-Markovian) has a $\sigma$-finite occupancy measure: since the environment and the policy are deterministic, every trajectory is the same. Since $\mathcal{A}=(0,1]$ and $p(s+a|s,a)=1$, the state $s_t$ is strictly increasing with $t$. This implies that either:
\begin{enumerate}
    \item the trajectory terminates and the occupancy measure is finite (and therefore $\sigma$-finite), 
    \item or the trajectory is upper bounded by 2 and by the monotone convergence theorem, it must converge to some state $s_\infty$ without ever reaching it (if $s_t=s_\infty$ for some $t$, the $s_{t+1}>s_\infty$, which is a contradiction).
\end{enumerate}
Since case 1. proves our point, we focus on case 2. from now. Still from the strictly increasing property, we infer that the occupancy measure of $\pi_{d}$ is 1 for the states $s_t$ on the deterministic trajectory and 0 everywhere else. We consider the following partition of $\mathcal{S}$:
\begin{align}
    \sigma_0 \doteq [s_\infty,2] \quad\quad\quad \forall i>0, \quad \sigma_i \doteq [s_{i-1},s_i).
\end{align}
By construction, $\mu^{\pi_d}_\gamma(\sigma_0) = 0$, $\forall i>0, \mu^{\pi_d}_\gamma(\sigma_i) = 1$, and $\mathcal{S}=\bigcup_{i\in\mathbb{N}}\sigma_i$, which proves the $\sigma$-finiteness of $\mu^{\pi_d}_\gamma$.

Now, we construct a policy $\pi$ that is $\sigma$-infinite. 
\begin{align}
    \pi(\cdot|t=0) = \mathcal{U}([0,1]) \quad\quad\quad \pi(\cdot|t>0) = \frac{1}{t} - \frac{1}{t+1}
\end{align}
Let $A_0$ denote the first action, which is the only stochastic one, then the state reached at time $t$ is:
\begin{align}
    S_{t+1} = S_{t} + \frac{1}{t} - \frac{1}{t+1} = A_0 + \sum_{t'=1}^{t} \frac{1}{t'} - \frac{1}{t'+1} = A_0 + 1 - \frac{1}{t+1},
\end{align}
which converges to $A_0+1$ as $t$ tends to infinity. For any segment $[b,c]\subset [0,1]$ with $b<c$, $\mathbb{P}(A_0\in(b,c])=c-b>0$. Then, we look at the measure of $[b+1,c+1]$:
\begin{align}
    \mu^{\pi}_\gamma([b+1,c+1]) &\geq \mathbb{P}\left(A_0\in(b,c]\right) \mathbb{E}\left[\sum_{t=0}^\infty \mathds{1}(s_t\in[b+1,c+1])\;\bigg|\; A_0\in(b,c]\right]\\
    &= \mathbb{P}\left(A_0\in(b,c]\right) \mathbb{E}\left[\sum_{t=\lceil\frac{1}{A_0-b}\rceil}^\infty 1\;\bigg|\; A_0\in(b,c]\right] = \infty,
\end{align}
which concludes the proof that $\mu^{\pi}_\gamma$ is $\sigma$-infinite.